\documentclass[11pt]{article}
\usepackage[utf8]{inputenc}
\usepackage[margin=1in]{geometry}
\pdfoutput=1
\usepackage{ifthen}
\newboolean{usenatbib}
 
\setboolean{usenatbib}{false}
 
\usepackage[style=alphabetic,maxalphanames=4,minalphanames=3,maxcitenames=2,maxbibnames=99,giveninits=false,doi=false,url=false,natbib]{biblatex}
\DeclareDelimFormat[bib]{nametitledelim}{\addspace}
 
\usepackage{csquotes}
\usepackage{bm}
\usepackage{fullpage}
\usepackage[english]{babel}
\usepackage[T1]{fontenc}
\usepackage{enumitem}
\setlist{nosep}
\usepackage{microtype}
\usepackage{booktabs}
\usepackage{hhline}
\usepackage{makecell}
\usepackage{etoolbox}
\apptocmd{\sloppy}{\hbadness 10000\relax}{}{}
 
\usepackage{amsmath,amsthm,amssymb,mathtools,thmtools}
\usepackage{graphicx}
\usepackage{color}
\usepackage[dvipsnames]{xcolor}
\usepackage[textsize=scriptsize,disable]{todonotes}
\usepackage[colorlinks=true, allcolors=magneta]{hyperref}
\usepackage{algorithm, algorithmic}
\usepackage{xfrac}
\usepackage[normalem]{ulem}

\hypersetup{colorlinks,
            breaklinks=true,
            linkcolor=purple,
            citecolor=purple,
            urlcolor=magenta,
            linktocpage,
            plainpages=false,
            }
 
\usepackage{macros}
 
\usepackage{comment}
\usepackage{tabularx}
\usepackage{subfig}
\usepackage{array}

\title{Online Learning with Improving Agents: Multiclass, Budgeted Agents and Bandit Learners}
\usepackage{xcolor}
\usepackage{tikz}
\usetikzlibrary{trees}

\author{%
 Sajad Ashkezari\thanks{University of Waterloo, \texttt{sajad.ashkezari@uwaterloo.ca}.}\quad\quad\quad\quad Shai Ben-David \thanks{University of Waterloo, \texttt{shai@uwaterloo.ca}.}
}


\begin{document}

\maketitle

\begin{abstract}%
  We investigate the recently introduced model of learning with improvements, where agents are allowed to make small changes to their feature values to be warranted a  more desirable label. We extensively extend previously published results by providing combinatorial dimensions that characterize online learnability in this model, by analyzing the multiclass setup, learnability in a bandit feedback setup, modeling agents' cost for making improvements and more.
\end{abstract}


\section{Introduction}
With the proliferation of machine learning based decision making tools and their application to societal and personal domains, there has been growing interest in understanding the implications of the use of such tools on the behavior of the individuals influenced by their decisions. 

 One aspect of such implications is addressed under the title of \emph{strategic classification}. It addresses the possible manipulation of user's data aimed to achieve desirable classification by the decision algorithm. The research along this line concerns setups where the feature vectors available to the learner differ from the true instance feature vectors in a way that increases the likelihood of some desirable outcomes. For example, users sign up for a gym club, to appear to be healthier to an algorithm assigning life insurance rates. Strategic classification learning aims to develop learning algorithms that mitigate the effects of such data manipulations (\cite{hardt2016strategic}, \cite{ahmadi2023fundamental} \cite{ahmadi2024strategic}).

Another related line of research aims to design algorithms that incentivize users to change their behavior (and consequently their true attributes) in a direction that improves their label (\cite{MMH20} and more). In this setup, the designer of the algorithm wishes to incentivize the individuals to exercise more. 

This paper follows a recent line of work title \emph{Learning with improvements} (\cite{attias2025pac}, \cite{sharma2025conservative}), where the learner assumes that the affected agents do change their true attributes towards achieving a desired label, but the focus is on accurate prediction of the resulting classification (rather than incentivizing behavioral change).

We extend the earlier published work on this topic along several axes:
\begin{enumerate}
    \item While earlier work analyzed only the case of learning with finite hypothesis classes, we characterize learnability via a combinatorial dimension.
    Our dimension based analysis implies learnability, with explicit successful learners, for infinite classes as well (Section \ref{Binary}). We note that the mistake bound in this model is always upper bounded by the usual (no improvements) mistake bound, however, in some cases there is a big gap between the two (Observations 1 and  \ref{obs: infinite littlestone} in that section).
    \item We extend the scope of addressed by earlier works on learning with improving agents by analyzing the multiclass case. Namely, we consider setups in which there are (arbitrarily) many possible labels with a user-preference ordering over these labels. This reflects a situation like having one's work being evaluated for several appreciation levels (say, your paper may be accepted as a poster, for a spotlight talk, for a full oral presentation, or for best paper award ....) (Section \ref{Mlticlass_full}).

    \item Once we discuss the multiclass setup, there is a natural question  about the feedback provided to an online learner. When the labels are binary, a feedback indicating whether a predicted instance label is true or false, also reveals its true label. In contrast, with more than two possible labels, the distinction between full information feedback - revealing the correct label, and partial information setup where the feedback is restricted to correct/wrong label.
    In Section \ref{bandit_uwg} we analyze this `bandit' setup (that is irrelevant to the binary classification setup). We provide a combinatorial characterization of the optimal mistake bound for this setup, as well as describe the optimal (mistake minimizing) learner. In Subsection \ref{price_bandits} we analyze the price in terms of additional mistakes, of the learner having only limited bandit feedback. 
    \item An underlying feature of learning with improvements is the notion of an \emph{improvement graph} whose nodes are agents' feature vectors and edges correspond to the ability of an agent to shift their feature vector (say, ``paper with typos" to ``paper without typos"). We extend previous work by removing the requirement that these graphs have a bounded degree.
    \item Finally, in Section \ref{Fullfeedback_weighted}, we extend our investigation to the setup in which the agents incur a cost for improving their features.
\end{enumerate}

\textbf{Related Work}



We study learning with improvement setting initiated by \citet{attias2025pac} who studied the problem in the PAC setting. They assume each agent can improve their features to a set of allowed features, if doing so would help them get a more desirable prediction. They show a separation between PAC learning and PAC learning with improvement. They also show that for some classes, allowing improvement makes it possible to find classifiers that achieve zero error as opposed to arbitrarily small error. This problem was extended by \citet{sharma2025conservative} to the online setting where the agents are chosen adversarially. However, \citet{sharma2025conservative} only consider \textit{finite} and \textit{binary} hypothesis classes. Moreover, they assume the number of points each agent can improve to is bounded. In this work we first extend their results (in the online setting) to infinite hypothesis classes and general improvement sets. We also introduce a model for studying the multiclass hypothesis classes where we associate each label with a value and each improvement with a cost.

Some other prior work study the problem of incentivizing agents to improve \cite{MMH20, kleinberg2020classifiers, haghtalab20, shavit2020causal}. However, following \citep{attias2025pac, sharma2025conservative} we assume a fixed improvement set and try to minimize the classification error.

A widely studied problem, which is closely related to the classification with improvement setting, is strategic classification \citep{hardt2016strategic}. In this setting, the agents can game and manipulate their features to get a better prediction. However, unlike the improvement setting, the manipulation does not truly change their classification. A common framework for this problem is to have a \textit{manipulation graph} whose nodes represent agents (their features) and an edge between two features means the agents can manipulate their features from one node to the other \citep{Zhang2021IncentiveAwarePL,lechner2022learning,pmlr-v202-lechner23a,ahmadi2023fundamental,cohen2024learnability,ahmadi2024strategic}. Notably, \citet{ahmadi2024strategic} introduce a new dimension that characterizes the optimal mistake bound of strategic online classification for binary labels. While our dimension for the binary setting is similar to theirs, we extend their results in two ways. First, we consider multiclass setting. Second, they consider cost of manipulation to be zero for the nodes that are connected by an edge (which is w.l.o.g. in the binary setting, as otherwise we could simply remove the edges whose cost is larger than the difference in utility of labels 0 and 1).

Online learning was introduced in seminal work of \citet{littlestone1988learning}. The Littlestone dimension of a hypothesis class is defined as the maximum depth of a tree that it \textit{shatters}. The dimensions that we introduces in this work have similar prototype. Our results in the bandit setting, where the learning algorithm only learns whether its prediction was correct or not as opposed to receiving the correct label, build on results in classic setting \cite{daniely2015multiclass, pmlr-v76-long17a}.

\section{Notation and Setup}

We let $\X$ and $\Y$ denote the instance and label space, respectively. We let $\X$ be any discrete space. In the binary setting, we let $\Y=\{0,1\}$ and in the multiclass setting, $\Y$ can be any finite set. A hypothesis is a function from $\X$ to $\Y$. A hypothesis class $\H\subseteq \Y^\X$ is a set of hypotheses. We use $\H_{x,y}:=\{h\in\H:h(x)=y\}$ to denote the subset of hypotheses that label $x$ with $y$. We study the online learning with improvements on a graph $G=(V,E)$ \cite{sharma2025conservative}, which we call the \textit{improvement graph}. In this setting, each agent is originally represented by a feature vector in $\X$. Let $V=\X$ denote the nodes of the graph. An edge $(x,v)\in E$ means an agent whose original feature is $x$ can potentially improve its features to $v$ to get a more desirable prediction. For each node $x$, let $\Delta(x)=\{v:(x,v)\in E\}$ denote the set of its neighbors, which we also refer to as its \textit{improvement set}. We assume each node $x\in V$ of the graph has a self-loop, that is, $x\in \Delta(x)$. We consider both unweighted and weighted graphs. In the weighted case, we let $\cost:E\to \mathbb{R}^+$ be the weight function where for an edge $e=(x,v)\in E$, $\cost(e)$ is the cost for an agent to improve its features from $x$ to $v$. Here we assume for each $x\in V$, $\cost(x,x)=0$. In case of an unweighted graph, we assume there is no cost for moving, i.e., $\cost(e)=0, \forall e\in E$.  We also define a utility function, $\val:\Y\to\mathrm{\mathbb{R}}$, where for $y\in\Y$, $\val(y)$ represents how valuable $y$ is.

The online learning with improvement happens in rounds as follows:

At time $t=1, 2, \cdots$
    \begin{enumerate}
        \item The environment presents an agent $\supt{x}$ to the learner.
        \item The learner implements a hypothesis $\supt{\hat{h}}$.
        \item The agent ``best responds'' to $\supt{\hat{h}}$ and improves to $\supt{v}$.
        \item The environment selects $\supt{y}$.
        \item The learner incurs loss $\indicator{\supt{\hat{h}}(\supt{v})\neq \supt{y}}$.
        \item The Learner receives ``some feedback'' on its prediction.
    \end{enumerate}

We say the learning problem is realizable by $\H$, if there exists some unknown $h^*\in \H$, such that for all $T\in \mathbb{N}$ and for all $t\leq T$, $h^*(\supt{v})=\supt{y}$.

Here, best response is defined as follows: for any $h$ and $x$, the agent will improve its features to $v\in \Delta(x)$ to maximize $\val(h(v))-\val(h(x))-\cost(x,v)$. If there is no $v$ with $\val(h(v))-\val(h(x))-\cost(x,v) > 0$, the agent doesn't move.

In particular, prior work of \citet{sharma2025conservative} studies the binary case with unweighted graph. In this case, $\val(1) > \val(0)$ and for any $x$, if $h(x)=0$ and there is $v\in \Delta(x)$ with $h(v)=1$, the agent $x$ moves to $v$ (in case of multiple such $v$'s, we assume the agent chooses adversarially). Otherwise, the agent doesn't move. Let $\Delta^+_h(x) = \{x'\in \Delta(x):h(x')=1\}$. Then the improvement set of an agent $x$ w.r.t. a classifier $h$ is as follows:
$$\Delta_h(x) = 
\begin{cases}
    \Delta^+_h(x) &\text{ if } h(x)=0 \text{ and }\Delta^+_h(x) \neq \emptyset\\
    \{x\} &\text{ otherwise }\\
\end{cases}
$$

In what follows, we study different variations of the problem depending on the label space, the graph, and the type of feedback the learner receives. Unless otherwise stated, we assume the input sequence is realizable by $\H$.

\section{Binary Classes} \label{Binary}

In this section we focus on binary classes with $\val(1)>\val(0)$. We also assume the graph is unweighted. \citet{sharma2025conservative} studied this setting for finite classes and graphs with bounded degree. Specifically, they show an upper bound of $(\Delta_G+1)\cdot\log(|\H|)$ where $\Delta_G$ is the maximum degree of a node in $G$. However, this bound is not tight. In fact, we can readily improve this result to $\ldim{\H}$ which denotes the Littlestone dimension of $\H$ \citep{littlestone1988learning}. The Littlestone dimension is defined exactly as Definition~\ref{def: ildim} for $\Delta(x)=\{x\}$. It characterizes the optimal mistake bound in the online learning without improving agents \citep{littlestone1988learning}.

\begin{observation}
    Any learner $A$ for online learning without improvement can be converted to a learner $A_I$ for online learning with improvement. This can simply be done as follows: The learner receives $\supt{x}$ and implements $\supt{h}$ where $\supt{h}(\supt{x})=A(\supt{x})$ and $\supt{h}(x')=0$ for all $x'\neq \supt{x}$, $A$ is updated if made a mistake. Thus, the agents never change their features and $A_I$ makes a mistake if and only if $A$ makes a mistake.
\end{observation}

The above observation implies that the optimal mistake bound when allowing improvement is at most the optimal mistake bound when we don't. Since the later is exactly $\ldim{\H}$ \citep{littlestone1988learning}, the former is also bounded by the same term. This improves the previous results as $\ldim{\H} \leq \log(|\H|)$. However, as we next show, finite Littlestone dimension is not necessary for having a finite mistake bound.

\begin{observation}
    \label{obs: infinite littlestone}
    Let $\X=\bigcup_{i\in \mathbb{N}} \{x_i, x_i'\}$ such that $\Delta(x_i)=\{x_i,x_i'\}$ and $\Delta(x_i')=\{x_i'\}$. Let $\H \subseteq \{0,1\}^\X$ be such that each $h\in \H$ labels each $x_i'$ with 1 and $\H$ can produce any labeling on $\{x_1,x_2, \cdots\}$. Then $\H$ has infinite Littlestone dimension. However, a learner that at each round implements $g$ with $g(x_i)=0$ and $g(x'_i)=1$ for all $i$, does not make any mistake.
\end{observation}

We now generalize the previous results by defining a Littlestone-type dimension and showing that it characterizes the optimal number of mistakes. Interestingly, the dimension that we define here is closely related to Strategic Littlestone dimension \cite{ahmadi2024strategic}. In particular, our trees have the same structure, however, our notion of shattering differs. The intuition behind our definition is that to force a mistake, the adversary must be able to label points in the improvement set with 0 as otherwise the learner can always predict label 1 on those points and the agent will move there similar to what happens in Observation~\ref{obs: infinite littlestone}.

\begin{definition}[Improvement Littlestone Tree ($\ilt$)]
    An ILT is a tree whose nodes are labeled by $\X$ such that each node $x$ has a set of $|\Delta(x)|+1$ outgoing edges labeled by $(x,1)$ and $(v,0)$ for all $v\in \Delta(x)$ (note that $x\in \Delta(x))$. An $\ilt$ is said to be shattered by $\H$ if all root-to-leaf paths of the form $(u^{(1)},y^{(1)}), \cdots, (u^{(d)},y^{(d)})$ are realizable by $\H$. That is, there exists $h\in\H$ such that $h(\supt{u})=\supt{y}$ for $1\leq t \leq d$, where $d$ is the depth of the tree.
\end{definition}

\begin{definition}[Improvement Littlestone Dimension $\mathrm{ILdim}$]
    \label{def: ildim}
    Improvement Littlestone Dimension is defined as the maximum depth of an $\ilt$ shattered by $\H$. In case of an unbalanced tree, depth of the tree is the length of its shortest branch.
\end{definition}

\begin{theorem}
    The optimal number of mistakes in the realizable online learning with improvement setting for deterministic learners is $\ildim{\H}$.
\end{theorem}
\begin{proof}
    The lower bound simply follows by following the maximal tree shattered by $\H$. That is, the adversary starts with the root of the tree, $x$ and receives the learner's hypothesis $h$. If $\Delta_h(x)=\{x\}$, the adversary moves along the edge labeled with $(x, 1-h(x))$. Otherwise, it chooses any $v\in \Delta_h(x)$ and moves along the edge $(v, 0)$. It then presents the child along the selected edge and proceed similarly. By definition of the dimension, this process can continue for at least $\ildim{\H}$ rounds where in each round, a mistake is forced.
    
    We now show that Algorithm~\ref{alg:isoa} achieves this lower bound. We claim that each mistake reduces the dimension of version space by 1, and thus the algorithm makes at most $\ildim{\H}$ mistake since the dimension is nonnegative. We prove the claim for two types of mistakes separately:

    False positive: In this case, either the agent remained at $\supt{v} = \supt{x}$ or moved to some point labeled 1 by the learner's predictor. In the later case, we know exactly to which point the agent moved to (i.e., $\supt{v}$) because by construction, we only label this point in the neighborhood with 1. By prediction principle, we have $\ildim{\vs_{\supt{v}, 0}} < \ildim{\vs}$, so since in case of false positive we set the new version to be $\ildim{\vs_{\supt{v}, 0}}$, the claim holds.

    False negative: In this case, the agent has not moved (otherwise our prediction would have been positive). So $\supt{v}=\supt{x}$ and the new version space will be $\vs_{\supt{x}, 1}$. Furthermore by our prediction rule $\ildim{\vs_{v, 0}} \geq \ildim{\vs}$ for all $v\in \Delta(x)$ (in fact, they're equal as the dimension can only decrease). Now, assume for the sake of contradiction that $\ildim{\vs_{\supt{x}, 1}} \geq \ildim{\vs}$. Then $\vs_{\supt{x}, 1}$ and all $\vs_{v, 0}$ shatters trees of depth $\ildim{\vs}$. We create a new tree whose root is $\supt{x}$ and whose out going edges are labeled by $(\supt{x},1)$, and $(v,0)$ for $v\in \Delta(\supt{x})$ and each edge connects to the respective tree of depth $\ildim{\vs}$. Then it's easy to see this tree of depth $\ildim{\vs} + 1$ is shattered by $\vs$, which contradicts the maximality of $\ildim{\vs}$. Hence, it must be that $\ildim{\vs_{\supt{x}, 1}} < \ildim{\vs}$ and the dimension decreases.
\end{proof}

\begin{algorithm}[tb]
  \caption{ISOA}
  \label{alg:isoa}
  \begin{algorithmic}
    \STATE Initialize $\vs = \H$
    \FOR{$t=1, 2,\cdots,$}
    \STATE Receive $\supt{x}$
    \STATE Let $\supt{\tilde{h}}$ be defined on each node $x$ as follows:
    $$\supt{\tilde{h}}(x)=\begin{cases}
        1 &\quad \text{if} \quad \ildim{\vs_{x, 0}} < \ildim{\vs}\\
        0 &\quad \text{otherwise}
    \end{cases}$$
    \STATE Let $\Delta_{\supt{\tilde{h}}}^+ = \{x'\in \Delta(\supt{x}):\supt{\tilde{h}}(x')=1\}$
    \IF{$\supt{\tilde{h}}(\supt{x}) = 1$ or $|\Delta_{\supt{\tilde{h}}}^+| =0$}
    \STATE Implement $\supt{h} \gets \supt{\tilde{h}}$
    \STATE $\supt{v} \gets \supt{x}$
    \ELSE
    \STATE Pick an arbitrary $\supt{v}\in \Delta_{\supt{\tilde{h}}}^+$
    \STATE Implement $\supt{h}$ that disagrees with $\supt{\tilde{h}}$ only on $\Delta_{\supt{\tilde{h}}}^+ \backslash \{\supt{v}\}$
    \ENDIF
    \IF{made a mistake}
    \STATE $\vs = \vs_{\supt{v}, 1-\supt{h}(\supt{v})}$
    \ENDIF
    \ENDFOR
  \end{algorithmic}
\end{algorithm}

\section{Full Feedback Multiclass - Unweighted Graph} \label{Mlticlass_full}

Here we assume the improvement graph is unweighted and $\Y = \{z_1, \cdots, z_k\}$ such that the labels are sorted in terms of preference $z_1 < z_2<\cdots<z_k$. Here $z < z'$ means $\val(z) <\val(z')$. Similar to the binary case, the agents are rational such that they will improve their features if and only if it allows them to receive a more preferred label. In this section, we assume the learner receives the final feature of the agent and its true label upon making a mistake. Note that receiving the final feature of the agent is only for simplifying the algorithm and is without loss of generality because our learner only allows a single choice for the agent to improve to (as we saw in the binary case). We address the bandit setting where the learner doesn't receive the true label in the next section.

We define a new tree-based dimension similar to an $\ilt$. The intuition behind the tree that we'll define is as follows: on any point $x$, in order to force a mistake, the adversary must be able to produce two different labels on $x$ itself. Moreover, on any neighbor $v\in \Delta(x)$, the adversary must be able to force a mistake by either producing two different labels or the minimal label $z_1$. Note that in case the agent has moved to $v$, the learner definitely doesn't label $v$ with $z_1$ as otherwise the agent wouldn't have moved to $v$.

\begin{definition}[Multiclass ILT]
    A multiclass ILT, is a tree whose nodes are labeled by $\X$. Each internal node $x$ has the following set of edges:
    \begin{itemize}
        \item Two edges labeled by $(x,y_1)$ and $(x,y_2)$ with $y_1\neq y_2$.
        \item For each $v\in \Delta(x)-\{x\}$: either an edge labeled $(v, z_1)$ or two edges labeled $(v,y_3)$ and $(v,y_4)$ with $y_3\neq y_4$ and $y_3,y_4\neq z_1$.
    \end{itemize}
    An ILT is said to be shattered by $\H$ if all of its branches are realizable by $\H$. 
\end{definition}

\begin{definition}[Multiclass $\mathrm{ILdim}$]
    $\ildim{\H}$ is defined as the maximum depth of an ILT shattered by $\H$ where the depth of a tree is defined as the length of its shortest branch.
\end{definition}

\begin{lemma}
    \label{lemma: lower bound, multiclass}
    Any deterministic learner in the multiclass online learning with improvement makes at least $\ildim{\H}$ mistakes.
\end{lemma}
\begin{proof}
    The adversary picks an ILT with maximum depth. It starts with presenting the root of the tree, $x$, if the learners hypothesis $h$ is such that the agent won't move, the adversary moves along the edge labeled $(x,y)$ where $y\neq h(x)$. Otherwise, let $v$ be the node that the agent moves to. If there exists an edge labeled $(v,z_1)$ move along that edge, otherwise, move along the edge labeled $(v,y)$ with $y\neq h(v)$. In any case, the learner makes a mistake. The adversary then presents the child along the chosen edge and continues the process in the same way. This process can continue at least $\ildim{\H}$ rounds and thus the learner makes at least these many mistakes.
\end{proof}

\begin{lemma}
    \label{lemma: possibilities for ildim}
    Let $\vs \subseteq \H$. For each $x$, at least one of the following happens: $\max_y \ildim{\vs_{x,y}}$ is achieved at a unique $y$ (case 1) or there exists $v\in\Delta(x)-\{x\}$ such that $\max_y \ildim{\vs_{v,y}}$ is achieved at a unique $y\neq z_1$ (case 2), or at least one of these maximums are strictly smaller $\ildim{\vs}$ (case 3). 
\end{lemma}
\begin{proof}
    Assume otherwise. Let $L=\ildim{\vs}$. Since case 3 doesn't hold, for all $x'\in \Delta(x)$, $\max_y \ildim{\vs_{x',y}}=L$. Since case 1 and 2 don't hold, there are $y_1 \neq y_2$ such that $\ildim{\vs_{x,y_1}}=\ildim{\vs_{x,y_2}}=L$ and for each $v\in \Delta(x)-\{x\}$, either $\ildim{\vs_{v,z_1}}=L$ or there are $y_3\neq y_4$ such that $\ildim{\vs_{v,y_3}}=\ildim{\vs_{v,y_4}}=L$. However, this means $\vs$ can shatter an ILT of depth $L+1$, which contradicts the definition of $\ildim{\vs}$.
\end{proof}

\begin{algorithm}[tb]
  \caption{Multiclass ISOA}
  \label{alg:multiclass isoa}
  \begin{algorithmic}
    \STATE Initialize $\vs = \H$
    \FOR{$t=1, 2,\cdots,$}
    \STATE Receive $\supt{x}$
    \IF{$\exists x'\in \Delta(\supt{x})$: $\max_y \ildim{\vs_{x',y}} < \ildim{\vs}$}
    \STATE Set $\supt{h}(x')=z_k$
    \STATE Set $\supt{h}(v)=z_1$ for all $v\in\Delta(\supt{x})-\{x'\}$
    \ELSIF{$y = \argmax_y \ildim{\vs_{\supt{x}, y}}$ is unique}
    \STATE Set $\supt{h}(\supt{x})=y$
    \STATE Set $\supt{h}(v)=z_1$ for all $v\in\Delta(\supt{x})-\{\supt{x}\}$
    \ELSE
    \STATE Let $x'\in \Delta(\supt{x})$ be such that
    \STATE $y = \argmax_y \ildim{\vs_{x', y}}$ is unique and $y\neq z_1$
    \STATE Set $\supt{h}(x')=y$
    \STATE Set $\supt{h}(v)=z_1$ for all $v\in \Delta(\supt{x})-\{x'\}$
    \ENDIF
    \IF{made a mistake}
    \STATE Let $\supt{v}$ be the final features of the agent
    \STATE Receive the true label $\supt{y}$ for $\supt{v}$
    \STATE $\vs = \vs_{\supt{v}, \supt{y}}$
    \ENDIF
    \ENDFOR
  \end{algorithmic}
\end{algorithm}

\begin{theorem}
    \label{thm: realizable multiclass}
    The optimal number of mistakes achieved by deterministic learners in the multiclass online learning with improvement equals $\ildim{\H}$. Furthermore, this optimal number is achieved by Algorithm~\ref{alg:multiclass isoa}.
\end{theorem}
\begin{proof}
    The lower bound is proved in Lemma~\ref{lemma: lower bound, multiclass}. By Lemma~\ref{lemma: possibilities for ildim} at least one of the cases in Algorithm~\ref{alg:multiclass isoa} happen. Then by construction, each mistakes reduces the dimension of the version space by at least one. Since the dimension is nonnegative, the algorithm makes at most $\ildim{\H}$ mistakes.
\end{proof}

\section{Bandit Feedback - Unweighted Graph} \label{bandit_uwg}

Here we study the same setting as the previous section, however, we assume the learner only learns whether it has made a mistake or not and it won't get the true label. 

For any hypothesis class $\vs$, we define $\vs_{x\nrightarrow y}:=\{h\in \vs:h(x)\neq y\}$. Unlike the full feedback setting where the learner received the true label after each mistake, here the learner only learns that the label of some $x$ is not $y$. Thus, we can update the set of candidate hypothesis, the version space, from $\vs$ to $\vs_{x\nrightarrow y}$. Our algorithm will keep track of the version space and predicts such that each mistake reduces its \textit{dimension}. The intuition behind this dimension, which we shortly define, is that the adversary must balance between forcing a mistake and keeping the version space ``large enough''. On any point $x$ (ignoring improvement for now), to keep this balance and force a mistake, the adversary must for all possible labels $y$ that the learner predicts, $\vs_{x\nrightarrow y}$ is such that the game can continue for as many rounds as possible. We now formally define the tree and dimension and then show it characterizes the optimal mistake bound.

\begin{definition}[Bandit Multiclass ILT]
    A bandit multiclass ILT (BILT), is a tree whose nodes are labeled by $\X$ and each internal node $x$ has the following set of edges:
    \begin{itemize}
        \item $k$ edges labeled by $(x,z_i)$ for $i\in[k]$.
        \item For each $v\in \Delta(x)-\{x\}$, $k-1$ edges labeled by $(v,z_i)$ for $i\in [k]\backslash \{1\}$
    \end{itemize}
    A BILT is shattered by $\H$ if for any branch $(u_1,y_1), \cdots, (u_d,y_d)$ there is $h\in\H$ such that $h(u_i)\neq y_i$ for all $i\leq d$.
\end{definition}

\begin{definition}[Multiclass $\mathrm{BILdim}$]
    $\bildim{\H}$ is defined as the maximum depth of a BILT shattered by $\H$ where the depth of a tree is defined as the length of its shortest branch.
\end{definition}

\begin{theorem}
    The optimal mistake bound of deterministic learners in the multiclass online learning with improvement and bandit feedback for sequences realizable by $\H$ is $\bildim{\H}$.
\end{theorem}
\begin{proof}
    The lower bound is achieved by an adversary that follows a BILT of depth $\bildim{\H}$ shattered by $\H$. The adversary starts by presenting the root of the tree, $x$, to the learner and receiving its hypothesis $h$. Let agent features depending on the learner's hypothesis is some $v\in \Delta(x)$. The adversary then indicates to the learner that it made a mistake and moves along the edge $(v, h(v))$ and presents the child along that edge and continues the game in the same way. Since the tree is shattered, there is some $h^*\in \H$ that differs from the learners predictions in each round for $\bildim{\H}$  at least rounds.

    The upper bound can be achieved by Bandit ISOA (BISOA) which we describe next. Start with $\vs^{(1)}=\H$. At round $t\geq 1$, receive $\supt{x}$ and implement $\supt{h}$ as follows:

    Let $v\in \Delta(\supt{x})$ be such that $\min_{y\in Y_{v}} \bildim{\supt{\vs}_{v\nrightarrow y}} < \bildim{\supt{\vs}}$ where $Y_{\supt{x}}=\Y$ and $Y_u=\Y-\{z_1\}$ for $u\in\Delta(\supt{x})-\{\supt{x}\}$. Such $v$ must exist as otherwise $\supt{\vs}$ shatters a tree of depth $\bildim{\supt{\vs}}+1$. Let $\supt{\hat{y}}$ be the minimizer. Then let $\supt{h}(v)=\supt{\hat{y}}$ and $\supt{h}(u)=z_1$ for $u\in \Delta(\supt{x})-\{x'\}$. 
    
    If made a mistake, update $\vs^{(t+1)}=\supt{\vs}_{v\nrightarrow \supt{\hat{y}}}$, otherwise $\vs^{(t+1)}=\supt{\vs}$.

    By our prediction rule, each mistake reduces BILdim of the version space by at least 1, thus the number of mistakes is at most $\bildim{\H}$ since BILdim is nonnegative.
\end{proof}

\subsection{Price of Bandit Feedback} \label{price_bandits}
Here we assume the improvement graph has bounded degree $\Delta_G$, that is, for all $x\in V$, $|\Delta(x)|\leq \Delta_G$. We want to answer for any $\H$, how many more mistakes the optimal learning with bandit feedback makes compared to the optimal learner with full feedback. We assume the final features of the agent is known to the learner, as it can be inferred similar to the binary case. For a hypothesis $h$, we define $\Delta^-(h,x):=\{x'\in \Delta(x):h(x')=z_1\}$ and $\Delta^+(h,x) = \Delta(x)-\Delta^-(h,x)$. A learning algorithm in the full feedback setting can be denoted by $e:(\X\times \Y)^*\times \X \to \Y^\X$. For a set of learners (also referred to as experts) $\mathcal{E}$ with weights $w:\mathcal{E}\to \mathbb{R}^+$ and any subset $F\subseteq \mathcal{E}$, we define $w(F):=\sum_{e\in F}w(e)$. For any such $\mathcal{E}$, we define $\mathcal{E}_{x,y}:=\{e\in E:e(x)=y\}$. Moreover, for a learner $e$, we use $e\gets x,y$ to be the learner $e$ updated with $x$ and $y$. We're now ready to state our results.

\begin{lemma}
    \label{lemma: price of bandit}
    Algorithm~\ref{alg:bandit to full reduction} makes at most $\mathcal{O}(\Delta_G \cdot \ildim{\H} \cdot k\log(k))$ when given the multiclass $\mathrm{ISOA}$ (Algorithm~\ref{alg:multiclass isoa}) as input. Thus, the price of bandit feedback is at most $\mathcal{O}(\Delta_G\cdot k\log(k))$
\end{lemma}

\begin{proof}
    Let $\supt{W}:=w(\mathcal{E})$ be the total weight of experts at round $t$. We claim after each mistake $W^{(t+1)} \leq (1-\frac{1}{2k(\Delta_G+1)})\supt{W}$. We prove the claim for each type of mistake separately.

    Mistake type 1 (the if condition holds): Here $\supt{h}$ predicts $z_1$ on every point in the neighborhood and thus, the agent doesn't move. We show that a constant fraction of the experts have the same behavior. We have $w(\{e\in \mathcal{E}:|\Delta^+(\supt{x}, \supt{h})|=0\})= \supt{W} - w(\{e:\exists x'\in \Delta(\supt{x}),\exists y\neq z_1, e(x')=y\})$. By our prediction rule, we have for any $x'$ and $y\neq z_1$, $w(\mathcal{E}_{x',y}) < \frac{1}{k \cdot(\Delta_G+1)}\supt{W}$. By union bound, we have $w(\{e:\exists x'\in \Delta(\supt{x}),\exists y\neq z_1, e(x')=y\}) \leq \sum_{x', y\neq z_1}w(\mathcal{E}_{x',y}) < \frac{\Delta_G}{\Delta_G+1}\supt{W}$ and thus $w(\{e\in \mathcal{E}:|\Delta^+(\supt{x}, \supt{h})|=0\})\geq \frac{1}{\Delta_G+1}\supt{W}$. For each such $e$, we will have multiple $e_y$ as define in the algorithm, such that the sum of their weights will be $\frac{w_e}{2}$. Thus, $W^{(t+1)} \leq (1-\frac{1}{2(\Delta_G+1)})\supt{W}$.

    Mistake type 2: In this case the weight of the experts that will be updated is at least $\frac{1}{k(\Delta_G+1)}\supt{W}$ and thus by similar arguments, $W^{(t+1)}\leq (1-\frac{1}{2k(\Delta_G+1)})\supt{W}$.

    Thus, at round $t$, if the learner has made $N$ mistakes, we have $\supt{W}\leq \exp{(-\frac{N}{2k(\Delta_G+1)})}$. On the other hand, there is an expert that has always been updated by a realizable sequence (since in each update we guess all possible true labels). Thus, by previous guarantees, such expert $e^*$ makes at most $\ildim{\H}$ many mistakes and thus $w_{e^*}\geq (\frac{1}{2(k-1)})^\ildim{\H}$. Since the weights are nonnegative, we must have $(\frac{1}{2(k-1)})^\ildim{\H} \leq \exp{(-\frac{N}{2k(\Delta_G+1)})}$. Thus, $N \leq \mathcal{O}(\Delta_G \cdot \ildim{\H} \cdot k\log(k))$.

\end{proof}

\begin{algorithm}[tb]
  \caption{Bandit to Full Feedback Reduction}
  \label{alg:bandit to full reduction}
  \begin{algorithmic}
    \STATE Input: Full feedback learner $\A$
    \STATE Initialize $\mathcal{E} = \{\A\}, w_\A = 1$
    \FOR{$t=1, 2,\cdots,$}
    \STATE Receive $\supt{x}$
    \STATE If $\exists x'\in \Delta(\supt{x}): \max_{y\neq z_1} w(\mathcal{E}_{x',y}) \geq \frac{w(\mathcal{E})}{k \cdot(\Delta_G+1)}$
    \STATE Set $\supt{h}(x')=\argmax_{y\neq z_1} w(\mathcal{E}_{x',y})$
    \STATE Set $\supt{h}(v)=z_1$ for all $v\in\Delta(\supt{x})-\{x'\}$
    \STATE Else: $\forall x'\in \Delta(\supt{x}): \supt{h}(x')=z_1$.
    \IF{made a mistake}
    \STATE Let $\supt{v}$ be the final features of the agent
    \STATE If $\supt{v}= \supt{x}$ and $|\Delta^+(\supt{h},\supt{x})|=0$:
    \STATE $\forall e\in \mathcal{E}$ with $|\Delta^+(e,\supt{x})|=0$ and $\forall y\neq z_1$:
    \STATE Add $e_y = e \gets \supt{x}, y$ to $\mathcal{E}$
    \STATE Set $w_{e_y}=\frac{w_e}{2(k-1)}$
    \STATE Remove $e$ from $\mathcal{E}$
    \STATE Else, $\forall e$ with $e(\supt{v})=\supt{h}(\supt{v}), \forall y\neq \supt{h}(\supt{v})$
    \STATE Add $e_y = e \gets \supt{x}, y$ to $\mathcal{E}$
    \STATE Set $w_{e_y}=\frac{w_e}{2(k-1)}$
    \STATE Remove $e$ from $\mathcal{E}$
    \ENDIF
    \ENDFOR
  \end{algorithmic}
\end{algorithm}

\section{Full Feedback - Weighted Graph} \label{Fullfeedback_weighted}
Here we study the multiclass setting where there is a cost related to each move. Again we assume the label space $\Y=\{z_1, \cdots, z_k\}$ and that for all $i<j$, $\val(z_i)<\val(Z_j)$. In the unweighted setting, we took advantage of the fact that if the agent moved from $x$ to $v$ in response to the learner's hypothesis $h$, then it must be the case that $h(v)\neq z_1$. Here, we could make a similar observation: for any $y\in\Y$, if $\val(y)-\val(h(x))\leq \val(y)-\val(z_1)<\cost(x,v)$, we know that $h(v)\neq y$ as otherwise the agent wouldn't have moved to $v$ since the cost of moving is more than the gain in the value. Thus, the adversary can force a mistake by setting the label of $v$ to such $y$. With that in mind, we can now define our dimension. Define $Y_{x,v}:=\{y\in \Y: \val(y)-\val(z_1)<\cost(x,v)\}$.

\begin{definition}
    A weighted $\ilt$, $\wilt$, is a tree whose nodes are labeled by $V$ and each internal node $x$ has the following set of edges:
    \begin{itemize}
        \item Two edges labeled $(x,y_1)$ and $(x,y_2)$ with $y_1\neq y_2$.
        \item For any $v\in \Delta(x)-\{x\}$, either a single edge labeled with $(v,y)$ with $y\in Y_{x,v}$ or two edges labeled with $(v,y_3), (v,y_4)$ for some $y_3\neq y_4$ with $y_3,y_4\in \Y-Y_{x,v}$
    \end{itemize}
    The tree is shattered by $\H$ if each branch is realizable similar to $\mathrm{ILdim}$. The weighted improvement dimension of $\H$, $\wildim{\H}$, is the maximum depth of a tree shattered by $\H$.
\end{definition}

We can now prove a similar result to Lemma~\ref{lemma: possibilities for ildim}. Without loss of generality, we assume for each $x\in V$ and $v\in \Delta(x)$, $\val(z_k)-\val(z_1)>\cost(x,v)$. Otherwise, no labeling incentivizes the agent to move from $x$ to $v$ and we could simply remove $v$ from $\Delta(x)$.

\begin{lemma}
    \label{lemma: possibilities for wildim}
    Let $\vs \subseteq \H$. For each $x$, at least one of the following happens: $\max_y \wildim{\vs_{x,y}}$ is achieved at a unique $y$ (case 1) or there exists $v\in\Delta(x)-\{x\}$ such that $\max_y \wildim{\vs_{v,y}}$ is achieved at a unique $y\notin Y_{x,v}$ (case 2), or at least one of these maximums is strictly smaller $\wildim{\vs}$ (case 3).

    \begin{proof}
        Assume case 3 doesn't happen. If case 1 doesn't happen, then there are $y_1,y_2\in\Y$ such that $\wildim{\vs_{x,y_1}}=\wildim{\vs_{x,y_2}}=\wildim{\vs}$. If case 2 doesn't happen then for each $v$, either there is $y\in Y_{x,v}$ with $\wildim{\vs_{v,y}}=\wildim{\vs}$ or there are $y_3,y_4\notin Y_{x,v}$ such that $\wildim{\vs_{v,y_3}}=\wildim{\vs_{v,y_4}}=\wildim{\vs}$. These would all imply that $\vs$ shatters a tree of depth $\wildim{\vs}+1$, which is a contradiction. Thus, at least one of the cases must happen.
    \end{proof}
\end{lemma}

\begin{theorem}
    The optimal mistake bound of deterministic learners in the full feedback multiclass setting with weighted graph for sequences realizable by $\H$ equals $\wildim{\H}$.
    \begin{proof}
        The proof is similar to the proof for unweighted graphs. In particular, for the upper bound, we can adapt Algorithm~\ref{alg:multiclass isoa} according to the three cases in Lemma~\ref{lemma: possibilities for wildim}.
    \end{proof}
\end{theorem}

\section{Conclusion}
In this work we studied online learning with improving agents. 
Our work extends and improves results of previous works on this topic in many aspects; We presented instance optimal learners that work for any hypothesis class and any improvement graph. Furthermore, we introduced the multiclass version and also described instance optimal learners for both the full feedback and bandit feedback settings. Finally, we studied the setting where there is a cost function associated with each possible movement and the agents only move if the utility that they gain is more than the cost they pay for moving.

We now raise some questions that we find interesting and leave open for future work.

\begin{enumerate}
    \item In this paper we only focused on deterministic learners. Can randomized learners achieve better (expected) mistake bounds?
    \item We only focused on the realizable setting, where the labels are consistent with some hypothesis in the class (that is known to the learner). In what way can our results be extended to the \emph{agnostic setting} where we the adversarially provided labels are not restricted by the hypothesis class?
    \item We only considered the classification error. 
    One may care about other costs of learning. For example, even if the learner does not make a mistake, it could be that the agent had to change its feature vector in order to get a desirable label, while had the learner implemented the true classifier, the agent could have saved such a change (and still get the desirable label). Such a case can be viewed as incentivizing an unnecessary burden on the agent. It would be interesting to see what is the trade-off between achieving low classification error and minimizing the unnecessary burden.
    \item We assume the improvement graph is known to the learner. What can we achieve if we relax this assumption to only provide the learner with some limited prior knowledge about the agent's improvement graph?
    \item We assume the learner publishes its implemented classifier. Namely, the agents actions are driven by the learner's hypothesis. Is it possible to learn when the agent is only allowed to access past classifiers? Similar problem has been studied in the strategic classification \citep{shao2025should}.
\end{enumerate}

\section{Acknowledgments}
Shai Ben-David is supported by an NSERC Discovery Grant and a Canada CIFAR AI Chair.

\printbibliography
\end{document}